\documentclass[preprint,12pt]{elsarticle}
\usepackage{a4wide}
\usepackage{csquotes}
\usepackage{graphicx}
\usepackage{amsmath,amsthm}
\usepackage{amsfonts}
\usepackage{enumitem}
\usepackage{bm}
\newcommand{\argmax}{\operatornamewithlimits{argmax}}
\newcommand{\argmin}{\operatornamewithlimits{argmin}}
\newcommand{\setlim}{\operatornamewithlimits{setlim}}
\newcommand{\defeq}{\stackrel{\text{def}}{=}}

\newtheorem{thm}{Theorem}
\newtheorem{lemma}{Lemma}

\newtheorem{defi}{Definition}

\begin{document}
\begin{frontmatter}
\title{Risk-averse estimation, an axiomatic approach to inference, and Wallace-Freeman without MML}

\author[monash]{Michael Brand}
\ead{michael.brand@monash.edu}

	\address[monash]{Faculty of IT (Clayton), Monash University, Clayton, VIC 3800, Australia}

\begin{abstract}
We define a new class of Bayesian point estimators, which we
refer to as risk averse. Using this definition, we formulate axioms that
provide natural requirements for inference, e.g.\ in a scientific setting,
and show that for
well-behaved estimation problems the axioms uniquely characterise an
estimator. Namely, for estimation problems in which some parameter values have
a positive posterior probability (such as, e.g., problems with a discrete
hypothesis space), the axioms characterise Maximum A Posteriori (MAP)
estimation, whereas elsewhere (such as in continuous estimation) they
characterise the Wallace-Freeman estimator.

Our results provide a novel justification for the Wallace-Freeman estimator,
which previously was derived only as an approximation to the
information-theoretic Strict Minimum Message Length
estimator. By contrast, our derivation requires neither approximations nor
coding.
\end{abstract}
% \subclass{62F10 \and 62F15}
\end{frontmatter}

\section{Introduction}

One of the fundamental statistical problems is point estimation. In a
Bayesian setting, this can be described as follows. Let
$(\boldsymbol{x},\boldsymbol{\theta}) \in X \times \Theta$
be a pair of random variables with a known joint distribution that
assigns positive
probability / probability density to any $(x,\theta)\in X\times\Theta$.
Here, $x$ is known as the \emph{observation}, $X\subseteq\mathbb{R}^N$ as
\emph{observation space}, $\theta$ as the \emph{parameter} and
$\Theta\subseteq \mathbb{R}^M$ as
\emph{parameter space}. We aim to describe a function
$\hat{\theta}:X\to\mathbb{R}^M$ such that $\hat{\theta}(x)$ is our
``best guess'' for $\boldsymbol{\theta}$ given $\boldsymbol{x}=x$.

Such a problem appears frequently for example in scientific inference, where we
aim to decide on a theory that best fits the known set of experimental results.

The optimal choice of a ``best guess'' $\hat{\theta}(x)$ naturally depends on
our definition of ``best''. The most common Bayesian approach regarding this is
that used by Bayes estimators \citep{berger2013statistical}, which define
``best'' explicitly, by means of a loss function. This allows estimators to
optimally trade off different types of errors, based on their projected
costs.

In this paper, we examine the situation where errors of all forms are
extremely costly and should therefore be minimised and if possible avoided,
rather than factored in. The scientific scenario, where one aims to decide on
a single theory, rather than a convenient trade-off between multiple hypotheses,
is an example. We
define this scenario rigorously under the name \emph{risk-averse estimation}.

We show that for problems in which some $\theta$ values have a positive
posterior probability, the assumption of risk-averse estimation
is enough to uniquely characterise Maximum A Posteriori (MAP),
\[
\hat{\theta}_{\text{MAP}}(x)\defeq\argmax_\theta \mathbf{P}(\boldsymbol{\theta}=\theta|\boldsymbol{x}=x)=\argmax_\theta \mathbf{P}(\theta|x)
\]

Risk-averse estimation does not suffice alone, however, to uniquely
characterise a solution for continuous problems, i.e.\ problems where
the joint distribution of
$(\boldsymbol{x},\boldsymbol{\theta})$ can be described by a probability
density function $f=f^{(\boldsymbol{x},\boldsymbol{\theta})}$.
To do so, we introduce three additional axioms, two of which relate to
invariance to representation and the last to invariance to irrelevant
alternatives, which reflect natural requirements for a good inference
procedure, all of which are also met by MAP.

(Notably, the estimator that maximises the posterior probability density
$f(\theta|x)$, which in the literature is usually also named MAP,
does not satisfy invariance to representation. To avoid confusion, we
refer to it as $f$-MAP.)

We prove regarding our risk-aversion assumption and three additional axioms
that together (and only together) they do uniquely characterise a single
estimation
function in the continuous case, namely the Wallace-Freeman estimator (WF)
\citep{WallaceFreeman1987},
\[
\hat{\theta}_\text{WF}(x)\defeq\argmax_{\theta} \frac{f(\theta|x)}{\sqrt{|I_\theta|}},
\]
where $I_\theta$ is the Fisher information matrix \citep{lehmann2006theory},
whose $(i,j)$ element is the conditional expectation
\[
I_\theta(i,j)\defeq\mathbf{E}\left(\left(\frac{\partial \log f(\boldsymbol{x}|\theta)}{\partial \theta(i)}\right)\left(\frac{\partial \log f(\boldsymbol{x}|\theta)}{\partial \theta(j)}\right)\middle|\boldsymbol{\theta}=\theta\right).
\]

A scenario not covered by either of the above is one where
$\boldsymbol{\theta}$ is a continuous variable but $\boldsymbol{x}$ is
discrete. To handle this case, we introduce
a fourth axiom, relating to invariance to superfluous information.

This creates a set of four axioms that is both symmetric and aesthetic: two
axioms relate to representation invariance (one in parameter space, the other
in observation space), and two relate to invariance to irrelevancies
(again, one in each domain).

We show that the four axioms together (and only together) uniquely characterise
WF also in the remaining case.

The fact that our axioms uniquely characterise the
Wallace-Freeman estimator is in itself of interest, because this
estimator exists almost exclusively as part of Minimum Message Length (MML)
theory \citep{Wallace2005}, and even there is defined merely as a
computationally convenient approximation to Strict MML (SMML)
\citep{wallace1975invariant}, which MML theory considers to be the optimal
estimator, for information-theoretical reasons.

Importantly, because SMML is computationally intractable in all but the
simplest cases \citep{farrwallace2002}, it is generally not used directly,
and MML practitioners are encouraged instead to approximate it by MAP in the
discrete case and by WF in the continuous
(See, e.g., \cite{comley200511}, p.~268). Thus, MML's standard practice
coincides with what risk-averse estimation advocates.
However, in the case of risk-averse estimation, neither MAP nor WF is
an approximation. Rather, they are both optimal estimators in their own rights
(within their respective domains),
and the justifications given for them are purely Bayesian
and involve no coding theory.

Thus, risk-averse estimation provides a new theoretical foundation, unrelated
to MML, that explains the empirical success of the MML recipe, for which
recent examples include
\cite{sumanaweera2018bits,schmidt2016minimum,saikrishna2016statistical,jin2015two,kasarapu2014representing}.

\section{Background}

\subsection{Bayes estimation}\label{S:Bayes}

The most commonly used class of Bayesian estimators is Bayes estimators.
A \emph{Bayes estimator}, $\hat{\theta}_L$, is defined
over a \emph{loss function},
\[
L=L_{(\boldsymbol{x},\boldsymbol{\theta})}:\Theta\times\Theta\to\mathbb{R}^{\ge 0},
\]
where $L(\theta_1,\theta_2)$ represents the cost of choosing $\theta_2$ when
the true value of $\boldsymbol{\theta}$ is $\theta_1$. The estimator
chooses an estimate that minimises the expected
loss given the observation, $x$:
\[
\hat{\theta}_L(x)\defeq\argmin_{\theta\in\Theta} \mathbf{E}(L(\boldsymbol{\theta},\theta)|\boldsymbol{x}=x).
\]

We denote by $P_\theta=P_{\boldsymbol{x}|\boldsymbol{\theta}=\theta}$ the
distribution of $\boldsymbol{x}$ at $\boldsymbol{\theta}=\theta$, i.e.\ the
likelihood of $\boldsymbol{x}$ given $\theta$,
and assume for all estimation problems and loss functions
\begin{equation}\label{Eq:distinct}
L(\theta_1,\theta_2)=0 \Leftrightarrow \theta_1=\theta_2 \Leftrightarrow P_{\theta_1}=P_{\theta_2}.
\end{equation}

When the distribution of $\boldsymbol{x}$ is known to be continuous, we
denote by $f_\theta=f^{(\boldsymbol{x},\boldsymbol{\theta})}_\theta$ the
probability density function (pdf) of $P_\theta$, i.e.\ 
$f^{(\boldsymbol{x},\boldsymbol{\theta})}_\theta(x)\defeq f^{(\boldsymbol{x},\boldsymbol{\theta})}(x|\theta)$.
Throughout, where $\boldsymbol{x}$ is known to be continuous, we use
$f_{\theta}$ interchangeably with $P_{\theta}$, and, in general, pdfs
interchangeably with the distributions they represent, e.g.\ in notation
such as ``$\boldsymbol{x}\sim f$'' for ``$\boldsymbol{x}$ is a
random variable with distribution (pdf) $f$''.

We say that $L$ is \emph{discriminative} for an estimation problem
$(\boldsymbol{x},\boldsymbol{\theta})$ if for every $\theta\in\Theta$ and
every neighbourhood $B$ of $\theta$, the infima over
$\theta'\in\Theta\setminus B$ of both $L(\theta,\theta')$ and
$L(\theta',\theta)$ are positive.

Notably, Bayes estimators are invariant to a linear monotone increasing
transform in $L$. They may also be defined over a \emph{gain function},
$G$, where $G$ is the result of a monotone \emph{decreasing} affine
transform on a loss function.

Examples of Bayes estimators
are posterior expectation, which minimises qua\-drat\-ic loss, and MAP,
which minimises loss over the discrete metric.
In general, Bayes estimators such as posterior expectation may return a
$\hat{\theta}$ value that is not in $\Theta$. This demonstrates how their
trade-off of errors may make them unsuitable for a high-stakes
``risk-averse'' scenario.

\subsection{Set-valued estimators}

Before defining risk-averse estimation, we must make a note regarding
set-valued estimators.

Typically, estimators are considered as functions from the observation space
to (extended) parameter space, $\hat{\theta}:X\to\mathbb{R}^M$.
However, all standard
point estimators are defined by means of an $\argmin$ or an $\argmax$.
Such functions intrinsically allow the result to be a subset of $\mathbb{R}^M$,
rather than an element of $\mathbb{R}^M$.

We say that an estimator is a \emph{well-defined point estimator} for
$(\boldsymbol{x},\boldsymbol{\theta})$ if it returns a single-element set for
every $x\in X$, in which case we take this element to be its estimate.
Otherwise, we say it is a \emph{set estimator}. The set estimator, in turn,
is \emph{well-defined} on $(\boldsymbol{x},\boldsymbol{\theta})$ if it
does not return an empty set as its estimate for any $x\in X$.

All estimators discussed will therefore be taken to be set estimators, and the
use of point-estimator notation should be considered solely as
notational convenience.

We also define \emph{set limit} and use the notation
\[
\setlim_{k\to\infty} B_k,
\]
where $\left(B_k\right)_{k\in\mathbb{N}}$ is a sequence of sets with an
eventually bounded union (i.e., there exists a $k$, such that
$\bigcup_{i\ge k} B_i$ is bounded),
to mean the set $\Omega$ of elements $\omega$ for
which there exists a monotone increasing sequence of naturals $k_1,k_2,\ldots$
and a sequence $\omega_1,\omega_2,\ldots$, such that for each $i$,
$\omega_i\in B_{k_i}$
and
$\lim_{i\to\infty} \omega_i = \omega$.

\section{Risk-averse estimation}

The idea behind MAP is to maximise the posterior probability that the
estimated value
is the correct $\boldsymbol{\theta}$ value. In the continuous domain this
cannot hold verbatim, because all $\theta\in\Theta$ have probability
zero. Instead, we translate the notion into the continuous domain by
maximising the probability that the estimated value is \emph{essentially}
the correct value. The way to do this is as follows.

\begin{defi}
A continuously differentiable, monotone decreasing function,
$A:\mathbb{R}^{\ge 0}\to\mathbb{R}^{\ge 0}$, satisfying
\begin{enumerate}
\item $A(0)>0$,
\item $\exists a_0 \forall a\ge a_0, A(a_0)=0$,
\end{enumerate}
will be called an \emph{attenuation function}, and the minimal $a_0$ will
be called its \emph{threshold value}.
\end{defi}

\begin{defi}
Let $L$ be the loss function of a Bayes estimator and $A$ an
attenuation function.

We define a \emph{risk-averse estimator} over $L$ and $A$
to be the estimator satisfying
\[
\hat{\theta}_{L,A}(x)=\setlim_{k\to\infty} \hat{\theta}_k(x),
\]
where $\hat{\theta}_k$ is the Bayes estimator whose gain function is
\[
G_k(\theta_1,\theta_2)=A(kL(\theta_1,\theta_2)).
\]

By convention we will assume $A(0)=1$, noting
that this value can be set by applying a positive multiple to the gain
function, which does not affect the definition of the estimator.
\end{defi}

The rationale behind this definition is that we use a loss function, $L$, to
determine how similar or different $\theta_2$ is to $\theta_1$, and then use
an attenuation function, $A$, to translate this divergence into a gain
function, where a $1$ indicates an exact match and a $0$ that $\theta_2$ is
not materially similar to $\theta_1$. (Such a gain function is often referred to
as a similarity measure.) The parameter $k$ is then
used to contract the neighbourhood of partial similarity, to the point that
anything that is not ``essentially identical'' to $\theta_1$ according to the
loss function is considered a $0$. Note that this is done without distorting the
loss function, as $k$ merely introduces a linear multiplication over it, a
transformation that preserves not only the closeness ordering of pairs but
also the Bayes estimator defined on the scaled function.

In this way, the risk-averse estimator maximises the
probability that $\theta_2$ is essentially identical to $\theta_1$, while
preserving our notion, codified in $L$, of how various $\theta$ values
interrelate.

\section{Positive probability events}\label{S:discrete}

\begin{thm}\label{T:MAP}
Any risk-averse estimator, $\hat{\theta}_{L,A}$, regardless of its loss
function $L$ or its attenuation function $A$, satisfies for any $x$ in any
estimation problem $(\boldsymbol{x},\boldsymbol{\theta})$ in which there
exists a $\theta\in\Theta$ with a positive posterior probability that
\[
\hat{\theta}_{L,A}(x) \subseteq \hat{\theta}_{\text{MAP}}(x)
\]
and is a nonempty set, provided $L$ is discriminative for the estimation
problem. In particular, $\hat{\theta}_{L,A}$ is in all such cases a
well-defined set estimator, and where MAP is a well-defined point estimator,
so is $\hat{\theta}_{L,A}$, and
\[
\hat{\theta}_{L,A}=\hat{\theta}_{\text{MAP}}.
\]
\end{thm}

\begin{proof}
The risk-averse estimator problem is defined by
\begin{equation}\label{Eq:LAdist}
\hat{\theta}_{L,A}(x)=\setlim_{k\to\infty} \argmax_{\theta} \mathbf{E}\left[A(kL(\boldsymbol{\theta},\theta))\middle| x\right].
\end{equation}

Fix $x$, and let $V_k(\theta)=\mathbf{E}\left[A(kL(\boldsymbol{\theta},\theta))\middle| x\right]$.

Let $N_k(\theta)$ be the set of $\theta'$ for which $A(kL(\theta',\theta))$ is
positive.
The value of $V_k(\theta)$ is bounded from both sides by
\begin{equation}\label{Eq:Vbounds}
\mathbf{P}(\theta|x)\le V_k(\theta)\le \mathbf{P}(\boldsymbol{\theta}\in N_k(\theta)|x).
\end{equation}

Because, by discriminativity of $L$, for any neighbourhood $N$ of $\theta$
there is a $k$ value from which
$N_k(\theta)\subseteq N$, as $k$ goes to infinity both bounds
converge to $\mathbf{P}(\theta|x)$. So, this is the limit for $V_k(\theta)$.
Also, $V_k(\theta)$ is a monotone decreasing function of $k$.

The above proves that
\[
\argmax_{\theta}\lim_{k\to\infty} \mathbf{E}\left[A(kL(\boldsymbol{\theta},\theta))\middle| x\right]
\]
is the MAP solution. To show that it is also the limit of the argmax (i.e.,
when switching back to the order of the quantifiers in \eqref{Eq:LAdist}), we
need to show certain uniformity
properties on the speed of convergence, which is what the remainder of this
proof is devoted to.

Let $\Omega=\{\theta\in\Theta|\mathbf{P}(\theta|x)>0\}$,
and define an enumeration $\left(\theta_i\right)_i$ over $\Omega$,
where the $\theta_i$ values are sorted by descending $\mathbf{P}(\theta_i|x)$.
(Such an enumeration is not necessarily unique.)
If $\Omega$ is countably infinite, the
values of $i$ range in $\mathbb{N}$. Otherwise, it is a finite enumeration,
with $i$ in $\{1,\ldots,|\Omega|\}$.

Let $S$ be the set $\{\theta_i:1\le i\le n\}$ for which
$\mathbf{P}(\theta|x)$ attains its maximum value, $v_0$.

Let $v_1=\mathbf{P}(\theta_{n+1}|x)$.

Because we know that for all $\theta\notin\Omega$ $V_k(\theta)$ is monotone
decreasing and tending to zero, there is for each such $\theta$ a threshold
value, $k'$, such that if $k\ge k'$, $V_k(\theta)<v_0$.
When this is the case, $\theta$ can clearly no longer be part of the argmax
in \eqref{Eq:LAdist}. Let $\Theta_k$ be the subset of $\Theta$ not thus
excluded at $k$.

By discriminativity of $L$, for any subset $\Theta'$ such that
$\Omega\subseteq\Theta'\subseteq\Theta$ and any
$\theta\in\Theta\setminus\Theta'$ there is a threshold value $k'$ such that
for all $k\ge k'$, there is no $\theta'\in\Theta'$ such that
$\theta\in N_{k}(\theta')$.

Combining these two observations, let $\Theta'_k$ be the set of $\theta$ such
that there is some $\theta'$ in $\Theta_k$ for which
$\theta\in N_k(\theta')$. We conclude that as $k$ grows to infinity, the
probability $\mathbf{P}(\boldsymbol{\theta}\in\Theta'_k\setminus\Omega|x)$
tends to zero. In particular, there exists a threshold value, which we will
name $k^{*}$, for which this probability is lower than $v_0-v_1$.

Let $U$ be a set $\Theta'_{k^{*}}\setminus\{\theta_i:i\le m\}$, where $m\ge n$
is such that
\begin{equation}\label{Eq:sumv0v1}
\mathbf{P}(\boldsymbol{\theta}\in U|x)<v_0-v_1.
\end{equation}
The choice of $m$ is not unique. However, such an $m$ always exists.

Define $T=\{\theta_{n+1},\ldots,\theta_m\}$.
Importantly, sets $S$ and $T$ are both finite.

Let $\{B_i\}_{i=1,\ldots,m}$ be a set of neighbourhoods of
$\{\theta_i\}_{i=1,\ldots,m}$, respectively, such that no two neighbourhoods
intersect. Because this set of $\theta$ values is finite, there is a minimum
distance between any two $\theta$ and therefore such neighbourhoods exist.

For each $i\in 1,\ldots,m$, let
$\delta_i=\inf_{\theta'\in\Theta\setminus B_i} L(\theta_i,\theta')$.

Because $L$ is discriminative,
all $\delta_i$ are positive. Because this is a finite set,
$\delta_{\text{min}}=\min_i \delta_i$ is also positive.

Consider now values of $k$ which are larger than
$\max(a_0/\delta_{\text{min}},k^{*})$,
where $a_0$ is the attenuation function's threshold value.

Because we chose all $B_i$ to be without intersection, any $\theta\in\Theta$
can be in at most one $B_i$. For $k$ values as described, only $\theta$
values in $B_i$ can have $\theta_i\in N_k(\theta)$. In particular, each
$N_k(\theta)$ can contain at most one of $\theta_1,\ldots,\theta_m$.

By \eqref{Eq:sumv0v1}, any such neighbourhood that does not contain one of
$\theta_1,\ldots,\theta_n$, i.e.\ set $S$, the MAP solutions, has a
$V_k(\theta)$ value lower than $v_0$, the lower bound for $V_k(\theta_1)$
given in \eqref{Eq:Vbounds}. This is because for a $\theta\in\Theta_{k^{*}}$,
the total value from all elements
in $U$ can contribute, by construction, less than $v_0-v_1$, whereas the one
element from $T$ that may be in the same neighbourhood can contribute no
more than $v_1$. On the other hand, a $\theta\notin\Theta_{k^{*}}$ has
$V_k(\theta)<v_0$ by the definition of $\Theta_{k^{*}}$.

Therefore, any $\hat{\theta}_k \in \argmax_\theta V_k(\theta)$ must
have an $N_k(\hat{\theta}_k)$ containing exactly one of
$\theta_1,\ldots,\theta_n$.

Let us partition any sequence $\left(\hat{\theta}_k\right)_{k\in\mathbb{N}}$
of such elements $\hat{\theta}_k$
according to the element of $S$ contained in $N_k(\hat{\theta}_k)$,
discarding any subsequence that is finite.

Consider now only the subsequence $k_1,k_2,\ldots$ such that
$\theta_i\in N_{k_j}(\hat{\theta}_{k_j})$ for some fixed $i\le n$.

By the same logic as before, because $L$ is discriminative,
for any neighbourhood $B$ of $\theta_i$
$\inf_{\theta'\in\Theta\setminus B} L(\theta_i,\theta')>0$, and therefore
there exists a $K$ value such that for all $k_j\ge K$, if
$\theta_i\in N_{k_j}(\theta')$ then $\theta'\in B$.

We conclude, therefore, that $\hat{\theta}_{k_j}\in B$ for all
sufficiently large $j$. By definition, the $\hat{\theta}_{k_j}$ sequence
therefore converges to $\theta_i$, and the set limit of the entire sequence
is the subset of the MAP solution, $\{\theta_1,\ldots,\theta_n\}$, for which
such infinite subsequences $k_1,k_2,\ldots$ exist.

Because the entire sequence is infinite, at least one of the subsequences
will be infinite, hence the risk-averse solution is never the empty set.
\end{proof}

\section{The axioms}\label{S:axioms}

We now describe additional good properties satisfied by the MAP estimator
which make it suitable for scenarios such as scientific inference. These
natural desiderata will form axioms of inference, which we will then
investigate outside the discrete setting.

Our interest is in investigating inference and estimation in situations
where all errors are highly costly, and hence we begin with an implicit
``Axiom 0'' that all estimators investigated are risk averse.

Our remaining axioms are not regarding the estimators themselves, but rather
regarding what constitutes a reasonable loss function for such estimators.
We maintain that these axioms
can be applied equally in all situations in which loss functions are
used, such as with Bayes estimators.

In all axioms, our requirement is that the loss function
$L$ satisfies the specified conditions
for every estimation problem $(\boldsymbol{x},\boldsymbol{\theta})$, and
every pair of parameters $\theta_1$ and $\theta_2$ in parameter space.

As always, we take the parameter space to be $\Theta\subseteq\mathbb{R}^M$ and
the observation space to be $X\subseteq\mathbb{R}^N$.

\begin{description}[style=unboxed]
	\item[\textbf{Axiom 1: Invariance to Representation of Parameter Space (IRP)}]\ \newline
A loss function $L$ is said to satisfy IRP if
for every invertible, continuous, differentiable function
$F:\mathbb{R}^M \to \mathbb{R}^M$, whose Jacobian is defined and non-zero everywhere,
\[
L_{(\boldsymbol{x},\boldsymbol{\theta})}(\theta_1,\theta_2)
=L_{(\boldsymbol{x},F(\boldsymbol{\theta}))}(F(\theta_1),F(\theta_2)).
\]

\item[\textbf{Axiom 2: Invariance to Representation of Observation Space (IRO)}]\ \newline
A loss function $L$ is said to satisfy IRO if
for every invertible, piecewise continuous, differentiable function
$G:\mathbb{R}^N \to \mathbb{R}^N$, whose Jacobian is defined and non-zero everywhere,
\[
L_{(\boldsymbol{x},\boldsymbol{\theta})}(\theta_1,\theta_2)
=L_{(G(\boldsymbol{x}),\boldsymbol{\theta})}(\theta_1,\theta_2).
\]

\item[\textbf{Axiom 3: Invariance to Irrelevant Alternatives (IIA)}]\ \newline
A loss function $L$ is said to satisfy IIA if
$L_{(\boldsymbol{x},\boldsymbol{\theta})}(\theta_1,\theta_2)$ does not depend on
any detail of the joint distribution of $(\boldsymbol{x},\boldsymbol{\theta})$
(described in the continuous case by the pdf $f(x,\theta)$)
other than at $\boldsymbol{\theta}\in\{\theta_1,\theta_2\}$.
\item[\textbf{Axiom 4: Invariance to Superfluous Information (ISI)}]\ \newline
A loss function $L$ is said to satisfy ISI if for any random variable
$\boldsymbol{y}$ such that $\boldsymbol{y}$ is independent of
$\boldsymbol{\theta}$ given $\boldsymbol{x}$,
\[
L_{((\boldsymbol{x},\boldsymbol{y}),\boldsymbol{\theta})}(\theta_1,\theta_2)
=L_{(\boldsymbol{x},\boldsymbol{\theta})}(\theta_1,\theta_2).
\]
\end{description}

A loss function that satisfies both IRP and IRO
is said to be \emph{representation invariant}.

The conditions of representation invariance follow \cite{wallace1975invariant},
whereas IIA was first introduced in a game-theoretic context by
\cite{nash1950bargaining}.

The ISI axiom is one we need neither in the positive probability case
discussed above nor in the continuous case of the next section. However,
we will use it in the remaining case, of discrete observations with a
continuous parameter space.

\section{The continuous case}\label{S:continuous}

\subsection{Well-behaved problems}

We now move to the harder case, where the distribution of $\boldsymbol{\theta}$
is continuous and none of its values is assigned a positive posterior
probability. We refer to this as the
\emph{$\boldsymbol{\theta}$-continuous case}.

We begin our exploration by looking at the special sub-case where
the joint distribution of $(\boldsymbol{x},\boldsymbol{\theta})$ is given by
a probability density function $f=f^{(\boldsymbol{x},\boldsymbol{\theta})}$.
We refer to this as \emph{the continuous case}. Much of the machinery we
develop for the continuous case will be reused, however, in the next
section, where we discuss problems with a discrete $\boldsymbol{x}$ but a
continuous $\boldsymbol{\theta}$. For this reason, where possible, we
describe our results in this section in terminology more general than is
needed purely for handling the continuous case.

We show that in the continuous case
for any well-behaved estimation problem and well-behaved loss function $L$,
if $L$ satisfies the first three invariance axioms of Section~\ref{S:axioms},
any risk-averse estimator over $L$ equals the Wallace-Freeman estimator,
regardless of its attenuation function.

Note that unlike in the discrete $\boldsymbol{\theta}$ case, in the
$\boldsymbol{\theta}$-continuous case we restrict our analysis to
``well-behaved'' problems.
The reason for this is mathematical convenience and simplicity of
presentation.

In this section we define well-behavedness. The definition will be one we
will reuse for analysing also the $\boldsymbol{\theta}$-continuous case.
However, some well-behavedness requirements for continuous problems are
not meaningful for distributions with a discrete $\boldsymbol{x}$, so the
definition states explicitly how the requirements are reduced for the
more general $\boldsymbol{\theta}$-continuous case.

We refer to a continuous/$\boldsymbol{\theta}$-continuous estimation problem as
\emph{well-behaved} if it satisfies the following criteria.

\begin{enumerate}
\item For continuous problems: the function $f(x,\theta)$ is
piecewise continuous in $x$ and three-times continuously differentiable in
$\theta$. If, alternatively, $\boldsymbol{x}$ is discrete, we merely require
that for every $x$, $f(\theta|x)$ is three-times continuously differentiable in
$\theta$.
\item The set $\Theta$ is a compact closure of an open set.
\end{enumerate}

Additionally, we say that a loss function $L$ is well-behaved if it
satisfies the following conditions.

\begin{description}
\item[\textbf{Smooth:}] If $(\boldsymbol{x},\boldsymbol{\theta})$ is a
well-behaved continuous/$\boldsymbol{\theta}$-continuous estimation problem,
then the function
$L_{(\boldsymbol{x},\boldsymbol{\theta})}(\theta_1,\theta_2)$ is three times
differentiable in $\theta_1$ and these derivatives are continuous in
$\theta_1$ and $\theta_2$.
\item[\textbf{Sensitive:}] There exists at least one well-behaved
continuous/$\boldsymbol{\theta}$-continuous estimation problem
$(\boldsymbol{x},\boldsymbol{\theta})$ and at least one choice of $\theta_0$,
$i$ and $j$ such that
\[
\left.\frac{\partial^2 L_{(\boldsymbol{x},\boldsymbol{\theta})}(\theta,\theta_0)}{\partial\theta(i)\partial\theta(j)}\right\rvert_{\theta=\theta_0}\ne 0.
\]
\item[\textbf{Problem-continuous:}] (For continuous estimation problems only:)
$L$ is problem-con\-tin\-u\-ous (or ``$\mathcal{M}$-continuous''), in the sense that
if $\left((\boldsymbol{x}_i,\boldsymbol{\theta})\right)_{i\in\mathbb{N}}$
is a sequence of well-behaved continuous estimation problems, such
that for every $\theta\in\Theta$,
$\left(f^{(\boldsymbol{x}_i,\boldsymbol{\theta})}_\theta\right)\xrightarrow{\mathcal{M}} f^{(\boldsymbol{x},\boldsymbol{\theta})}_{\theta}$,
then for every $\theta_1,\theta_2\in\Theta$,
\[
\lim_{i\to\infty} L_{(\boldsymbol{x}_i,\boldsymbol{\theta})}(\theta_1,\theta_2)=L_{(\boldsymbol{x},\boldsymbol{\theta})}(\theta_1,\theta_2).
\]
\end{description}

In the last criterion, the symbol ``$\xrightarrow{\mathcal{M}}$'' indicates
convergence in measure \citep{halmos2013measure}. This is defined as follows.
Let $\mathcal{M}$ be the space of normalisable, non-atomic measures over some
$\mathbb{R}^s$, let $f$ be a function $f:\mathbb{R}^s\to\mathbb{R}^{\ge 0}$ and
let $\left(f_i\right)_{i\in\mathbb{N}}$ be a sequence of such
functions. Then $\left(f_i\right)\xrightarrow{\mathcal{M}} f$ if
\begin{equation}\label{Eq:measure0}
\forall\epsilon>0,\lim_{i\to\infty}\mu(\{x\in\mathbb{R}^s:|f(x)-f_i(x)|\ge \epsilon\})=0,
\end{equation}
where $\mu$ can be, equivalently, any measure in $\mathcal{M}$ whose support
is at least the union of the support of $f$ and all $f_i$.

We will usually take $f$ and all $f_i$ to be pdfs.
When this is the case, $\mu$'s support only needs to equal the support
of $f$. Furthermore,
because $\mu$ is normalisable, one can always choose values $a$ and
$b$ such that $\mu(\{x:0<f(x)<a\})$ and $\mu(\{x:f(x)>b\})$
are both arbitrarily small, for which reason one can substitute the absolute
difference ``$\ge\epsilon$'' in \eqref{Eq:measure0} with a relative difference
``$\ge\epsilon f(x)$'', and reformulate it in the case that $f$ and all
$f_i$ are pdfs as
\begin{equation}\label{Eq:measure_ratio}
\forall\epsilon>0,\lim_{i\to\infty}\mathbf{P}_{\boldsymbol{x}\sim f}(|f(\boldsymbol{x})-f_i(\boldsymbol{x})|\ge \epsilon f(\boldsymbol{x}))=0.
\end{equation}

This reformulation makes it clear that convergence in measure over pdfs
is a condition
independent of representation: it is invariant to transformations of the
sort we allow on the observation space.

\subsection{The main theorem}

Our main theorem for continuous problems is as follows.

\begin{thm}\label{T:main}
If $(\boldsymbol{x},\boldsymbol{\theta})$ is a well-behaved continuous
estimation problem for which $\hat{\theta}_{\text{WF}}$ is a well-defined
set estimator, and if $L$ is a well-behaved loss function, discriminative
for $(\boldsymbol{x},\boldsymbol{\theta})$,
that satisfies all of IIA, IRP and IRO,
then any risk-averse estimator
$\hat{\theta}_{L,A}$ over $L$, regardless of its attenuation function $A$,
is a well-defined set estimator, and for every $x$,
\[
\hat{\theta}_{L,A}(x)\subseteq \hat{\theta}_{\text{WF}}(x).
\]

In particular, if $\hat{\theta}_{\text{WF}}$ is a well-defined point estimator,
then so is $\hat{\theta}_{L,A}$, and
\[
\hat{\theta}_{L,A}=\hat{\theta}_{\text{WF}}.
\]
\end{thm}

We prove this through a progression of lemmas.
For the purpose of this derivation, the dimension of the parameter space, $M$,
and the dimension of the observation space, $N$, are throughout taken to be
fixed, so as to simplify notation.

\begin{lemma}\label{L:likelihood}
For estimation problems $(\boldsymbol{x},\boldsymbol{\theta})$ with a
continuous $\boldsymbol{\theta}$,
if $L$ satisfies both IIA and IRP then
$L_{(\boldsymbol{x},\boldsymbol{\theta})}(\theta_1,\theta_2)$ is a function
only of the likelihoods $P_{\theta_1}$ and
$P_{\theta_2}$.
\end{lemma}

\begin{proof}
The IIA axiom is tantamount to stating that
$L_{(\boldsymbol{x},\boldsymbol{\theta})}(\theta_1,\theta_2)$ is dependent
only on the following:
\begin{enumerate}
\item the function's inputs $\theta_1$ and $\theta_2$,
\item the likelihoods $P_{\theta_1}$ and $P_{\theta_2}$, and
\item the priors $f(\theta_1)$ and $f(\theta_2)$.
\end{enumerate}

We can assume without loss of generality that $\theta_1\ne\theta_2$, or else
the value of $L(\theta_1,\theta_2)$ can be determined to be zero by
\eqref{Eq:distinct}.

Our first claim is that, due to IRP, $L$ can also not depend on the problem's
prior probability densities $f(\theta_1)$ and $f(\theta_2)$. To show this,
construct an invertible, continuous, differentiable function
$F:\mathbb{R}^M \to \mathbb{R}^M$, whose Jacobian is defined and non-zero
everywhere, in the following way.

Let $(\vec{v}_1,\ldots,\vec{v}_M)$ be an orthogonal basis
for $\mathbb{R}^M$ wherein $\vec{v}_1=\theta_2-\theta_1$.
We design $F$ as
\[
F\left(\theta_1+\sum_{i=1}^M b_i\vec{v}_i\right)=\theta_1+F_1(b_1)\vec{v}_1+\sum_{i=2}^M b_i\vec{v}_i,
\]
where $F_1:\mathbb{R}\to\mathbb{R}$ is a continuous, differentiable function
onto $\mathbb{R}$, with a derivative that is positive everywhere, satisfying
\begin{enumerate}
\item $F_1(0)=0$ and $F_1(1)=1$, and
\item $F_1'(0)=d_0$ and $F_1'(1)=d_1$, for some arbitrary
positive values $d_0$ and $d_1$.
\end{enumerate}

Such a function is straightforward to construct for any values of $d_0$ and
$d_1$, and by an appropriate choice of these values, it is possible to map
$(\boldsymbol{x},\boldsymbol{\theta})$ into
$(\boldsymbol{x},F(\boldsymbol{\theta}))$ in a way that does not change
$P_{\theta_1}$ or $P_{\theta_2}$, but adjusts $f(\theta_1)$ and $f(\theta_2)$
to any desired positive values.

Lastly, we show that $L(\theta_1,\theta_2)$ can also not depend on the values
of $\theta_1$ and $\theta_2$ other than through $P_{\theta_1}$ and
$P_{\theta_2}$. For this
we once again invoke IRP: by applying a similarity transform on $\Theta$, we
can map any $\theta_1$ and $\theta_2$ values into arbitrary new values, again
without this affecting their respective likelihoods.
\end{proof}

In light of Lemma~\ref{L:likelihood}, we will henceforth use the notation
$L(P_{\theta_1},P_{\theta_2})$ (or, when $\boldsymbol{x}$ is known to be
continuous, $L(f_{\theta_1},f_{\theta_2})$) instead of
$L_{(\boldsymbol{x},\boldsymbol{\theta})}(\theta_1,\theta_2)$.
A loss function that can be written in this way is referred to
as a \emph{likelihood-based} loss function.

For distributions $P$ and $Q$ with a common support $X\subseteq\mathbb{R}^N$,
absolutely continuous with respect to each other,
let $r_{P,Q}(x)$ be the Radon-Nikodym derivative
$\frac{\text{d}P}{\text{d}Q}(x)$ \citep{royden1988real}.
For a value of $x$ that has positive probability in both $P$ and $Q$, this is
simply $\mathbf{P}_{\boldsymbol{x}\sim P}(\boldsymbol{x}=x)/\mathbf{P}_{\boldsymbol{x}\sim Q}(\boldsymbol{x}=x)$,
whereas for pdfs $f$ and $g$ it is $f(x)/g(x)$ within the common support.

We now define the function $c[P,Q]:[0,1]\to\mathbb{R}^{\ge 0}$ by
\[
c[P,Q](t)\defeq\inf\left(\{r\in\mathbb{R}^{\ge 0}:t\le \mathbf{P}_{\boldsymbol{x}\sim Q}(r_{P,Q}(\boldsymbol{x})\le r)\}\right).
\]

\begin{lemma}\label{L:ccont}
The function $c$ is $\mathcal{M}$-continuous for continuous distributions,
in the sense that if both
$\left(f_i\right)_{i\in\mathbb{N}}\xrightarrow{\mathcal{M}} f$ and
$\left(g_i\right)_{i\in\mathbb{N}}\xrightarrow{\mathcal{M}} g$,
where $f$ and $g$ are pdfs and $\left(f_i\right)_{i\in\mathbb{N}}$ and
$\left(g_i\right)_{i\in\mathbb{N}}$ are pdf sequences,
then $\left(c[f_i,g_i]\right)_{i\in\mathbb{N}}\xrightarrow{\mathcal{M}} c[f,g]$.
\end{lemma}

\begin{proof}
For any $t_0\in [0,1]$, let $r_0=c[f,g](t_0)$, and let
$t_{\text{min}}$ and $t_{\text{max}}$ be the infimum $t$ and the
supremum $t$, respectively, for which $c[f,g](t)=r_0$.

Because $c[f,g]$ is a monotone increasing function,
$t_{\text{min}}$ is also the supremum $t$ for which
$c[f,g](t)< r_0$ (unless no such $t$ exists, in which case $t_{\text{min}}=0$),
so by definition $t_{\text{min}}$ is the supremum of
$\mathbf{P}_{\boldsymbol{x}\sim g}(r_{f,g}(\boldsymbol{x})\le r)$,
for all $r<r_0$, from which we conclude
\[
t_{\text{min}}=\mathbf{P}_{\boldsymbol{x}\sim g}(r_{f,g}(\boldsymbol{x})<r_0).
\]

Because $\left(f_i\right)\xrightarrow{\mathcal{M}} f$ and
$\left(g_i\right)\xrightarrow{\mathcal{M}} g$, we can use
\eqref{Eq:measure_ratio} to determine that using a large enough $i$ both
$f_i(x)/f(x)$ and $g_i(x)/g(x)$ are arbitrarily close to $1$ in all
but a diminishing measure of $X$. Hence,
\begin{equation*}
\begin{aligned}
\lim_{i\to\infty} \mathbf{P}_{\boldsymbol{x}\sim g_i}(r_{f_i,g_i}(\boldsymbol{x})<r_0)
&=\lim_{i\to\infty} \mathbf{P}_{\boldsymbol{x}\sim g}(r_{f_i,g_i}(\boldsymbol{x})<r_0) \\
&=\mathbf{P}_{\boldsymbol{x}\sim g}(r_{f,g}(\boldsymbol{x})<r_0) \\
&=t_{\text{min}}.
\end{aligned}
\end{equation*}

We conclude that for any $t^{+}>t_{\text{min}}$ a large enough $i$ will satisfy
\[
\mathbf{P}_{\boldsymbol{x}\sim g_i}(r_{f_i,g_i}(\boldsymbol{x})<r_0)< t^{+},
\]
and hence
$c[f_i,g_i](t^{+})\ge r_0$. For all such $t^{+}$, and in particular for all
$t^{+}>t_0$,
\begin{equation}\label{Eq:liminf}
\liminf_{i\to\infty} c[f_i,g_i](t^{+})\ge r_0 = c[f,g](t_0).
\end{equation}

A symmetrical analysis on $t_{\text{max}}$ yields that for all $t^{-}<t_0$,
\begin{equation}\label{Eq:limsup}
\limsup_{i\to\infty} c[f_i,g_i](t^{-})\le c[f,g](t_0).
\end{equation}

Consider, now, the functions
\[
c_{\text{sup}}(t)\defeq\limsup_{i\to\infty} c[f_i,g_i](t)
\]
and
\[
c_{\text{inf}}(t)\defeq\liminf_{i\to\infty} c[f_i,g_i](t).
\]
Because each
$c[f_i,g_i]$ is monotone increasing, so are $c_{\text{sup}}$ and
$c_{\text{inf}}$. Monotone functions can only have countably many discontinuity
points (for a total of measure zero). For any $t_0$ that is not
a discontinuity point of either
function, we have from \eqref{Eq:liminf} and \eqref{Eq:limsup}
that $\lim_{i\to\infty} c[f_i,g_i](t_0)$ exists and equals $c[f,g](t_0)$,
so the conditions of convergence in measure hold.
\end{proof}

\begin{lemma}\label{L:r}
If $L$ satisfies IRO and is a well-behaved likelihood-based loss function and
$p$ and $q$ are piecewise-continuous probability density functions over
$X\subseteq\mathbb{R}^N$, then
$L(p,q)$ depends only on $c[p,q]$.
\end{lemma}

\begin{proof}
The following conditions are equivalent.
\begin{enumerate}
\item $c[p,q]$ equals the indicator function on $(0,1]$ in all but a measure
zero of values,
\item $p$ equals $q$ in all but a measure zero of $X$,
\item $p$ and $q$ are $\mathcal{M}$-equivalent, in the sense that a sequence
of elements all equal to $p$ nevertheless satisfies the condition of
$\mathcal{M}$-convergence to $q$, and
\item $L(p,q)=0$,
\end{enumerate}
where the equivalence of the last condition follows from the previous one by
problem continuity, together with \eqref{Eq:distinct}.
Hence, if the second condition is met, we are done.
We can therefore assume that $p$ and $q$ differ in a positive measure of $X$,
and (because
both integrate to $1$) that they are consequently not linearly dependent.

Because $L$ is known to be likelihood-based, the value of
$L(p,q)$ is not dependent on the full details of the
estimation problem: it will be the same in any estimation problem of the
same dimensions that contains the likelihoods $p$ and $q$.
Let us therefore design an estimation problem that is easy to analyse but
contains these two likelihoods.

Let $(\boldsymbol{x},\boldsymbol{\theta})$ be an estimation problem with
$\Theta=[0,1]^M$ and a uniform prior on $\boldsymbol{\theta}$. Its likelihood
at $\theta_0=(0,\ldots,0)$ will be $p$, at $\theta_1=(1,0,\ldots,0)$
will be $q$,
and we will choose piecewise continuous likelihoods, $f_\theta$, over the rest
of the $\theta\in\{0,1\}^M$ so all $2^M$ are linearly independent, share the
same support, and differ
from each other over a positive measure of $X$, and so that their
respective $r_{f_\theta,q}$ values are all monotone weakly increasing
with $r_{p,q}$ and with each other.

If $M>1$, we further choose $f_{\theta_2}$ at $\theta_2=(0,1,0,\ldots,0)$ to
satisfy that $c[f_{\theta_2},q]$ is
monotone strictly increasing.
If $M=1$, this is not necessary and we, instead, choose $\theta_2=\theta_0$. 

We then extend this description
of $f^{(\boldsymbol{x},\boldsymbol{\theta})}_\theta$ at $\{0,1\}^M$ into a full
characterisation of all the problem's likelihoods by setting these to be
multilinear functions of the coordinates of $\theta$.

We now create a sequence of estimation problems,
$(\boldsymbol{x}_i,\boldsymbol{\theta})$ to satisfy the conditions of
$L$'s problem-continuity assumption. We do this by constructing a sequence
$\left(S_i\right)_{i\in\mathbb{N}}$ of subsets of $\mathbb{R}^N$ such that
for all $\theta\in\{0,1\}^M$,
$\mathbf{P}(\boldsymbol{x}\in S_i|\boldsymbol{\theta}=\theta)$ tends to $1$,
and for every $x\in S_i$,
\begin{equation}\label{Eq:epsilon_i}
\left|f^{(\boldsymbol{x},\boldsymbol{\theta})}_\theta(x)-f^{(\boldsymbol{x}_i,\boldsymbol{\theta})}_\theta(x)\right|<\epsilon_i,
\end{equation}
for an arbitrarily-chosen sequence $\left(\epsilon_i\right)_{i\in\mathbb{N}}$
tending to zero.
By setting the remaining likelihood values as multilinear functions of the
coordinates of $\theta$, as above, the sequence
$(\boldsymbol{x}_i,\boldsymbol{\theta})$ will satisfy the problem-continuity
condition and will guarantee
$\lim_{i\to\infty} L_{(\boldsymbol{x}_i,\boldsymbol{\theta})}(\theta_0,\theta_1)=L_{(\boldsymbol{x},\boldsymbol{\theta})}(\theta_0,\theta_1)=L(p,q)$.

Each $S_i$ will be describable by the positive parameters $(a,b,d,r)$ as
follows.
Let $C^N_d=\{x\in\mathbb{R}^N:|x|_{\infty}\le d/2\}$, i.e.\ the axis parallel,
origin-centred, $N$-dimensional cube of side length $d$.
$S_i$ will be chosen to contain all $x\in C^N_d$
such that for all $\theta\in\{0,1\}^M$,
$a\le f^{(\boldsymbol{x},\boldsymbol{\theta})}_\theta(x)\le b$ and $x$ is at
least a distance of $r$ away from the nearest discontinuity point of
$f^{(\boldsymbol{x},\boldsymbol{\theta})}_\theta$, as well as from the origin.
By choosing small enough
$a$ and $r$ and large enough $b$ and $d$, it is always possible to make
$\mathbf{P}(\boldsymbol{x}\in S_i|\boldsymbol{\theta}=\theta)$ arbitrarily
close to $1$, so the sequence can be made to satisfy its requirements.

We will choose $d$ to be a natural.

We now describe how to construct each
$f^{(\boldsymbol{x}_i,\boldsymbol{\theta})}_\theta$ from its respective
$f^{(\boldsymbol{x},\boldsymbol{\theta})}_\theta$.
We first describe for each $\theta\in\{0,1\}^M$
a new function $g^i_\theta:\mathbb{R}^N\to\mathbb{R}^{\ge 0}$
as follows. Begin by setting
$g^i_\theta(x)=f^{(\boldsymbol{x},\boldsymbol{\theta})}_\theta(x)$
for all $x\in S_i$.
If $x\notin C^N_d$, set $g^i_\theta(x)$
to zero. Otherwise, complete the $g^i_\theta$ functions so that all are
linearly independent and so that each
is positive and continuous inside $C^N_d$, and integrates to $1$.
Note that because a neighbourhood around the origin is known to not be in
$S_i$, it is never the case that $S_i=C^N_d$.
This allows enough degrees of
freedom in completing the functions $g$ in order to meet all
their requirements.

As all $g^i_\theta$ are continuous
functions over the compact domain $C^N_d$,
by the Heine-Cantor Theorem \citep{rudin1964principles} they are
uniformly continuous. There must therefore exist a natural $n$, such that we can
tile $C^N_d$ into sub-cubes of side length $1/n$ such that by setting each
$f^{(\boldsymbol{x}_i,\boldsymbol{\theta})}_\theta$ value in each sub-cube to a
constant for the sub-cube equal to the mean over the entire sub-cube tile
of $g^i_\theta$, the result will satisfy for all $x\in C^N_d$ and all
$\theta\in\{0,1\}^M$,
$\left|f^{(\boldsymbol{x}_i,\boldsymbol{\theta})}_\theta(x)-g^i_\theta(x)\right|<\epsilon_i$.
Because $f^{(\boldsymbol{x}_i,\boldsymbol{\theta})}$ is by design multi-linear
in $\theta$, this implies that for all
$\theta\in [0,1]^M$ and all $x\in S_i$, condition
\eqref{Eq:epsilon_i} is attained.
Furthermore, by choosing a large enough $n$, we can always ensure, because
the $g$ functions are continuous and linearly independent, that also the
$f^{(\boldsymbol{x}_i,\boldsymbol{\theta})}_\theta$ functions, for
$\theta\in \{0,1\}^M$ are linearly independent and differ in more than a
measure zero of $\mathbb{R}^N$. Together, these properties ensure that the
new problems constructed are both well defined and well behaved.

We have therefore constructed $(\boldsymbol{x}_i,\boldsymbol{\theta})$ as
a sequence of well-behaved estimation problems that $\mathcal{M}$-approximate
$(\boldsymbol{x},\boldsymbol{\theta})$ arbitrarily well, while being entirely
composed of $f^{(\boldsymbol{x}_i,\boldsymbol{\theta})}_\theta$ functions whose
support is $C^N_{d_i}$ for some natural
$d_i$ and whose values within their support are piecewise-constant inside cubic
tiles of side-length $1/n_i$, for some natural $n_i$.

We now use IRO to reshape the observation space of the estimation problems
in the constructed sequence by a piecewise-continuous transform.

Namely, we take each constant-valued cube of side length $1/n_i$ and
transform it using a scaling transformation in each coordinate,
as follows. Consider a single cubic tile, and let
the value of $f^{(\boldsymbol{x}_i,\boldsymbol{\theta})}_{\theta_1}(x)$
at points $x$ that are within it be $G_i$.
We scale the first coordinate of the tile to be of length
$G_i/n_i^N$, and all other coordinates to be of length $1$.
Notably, this transformation increases the volume of the cube by a factor
of $G_i$, so the probability density inside the cube, for each $f_\theta$,
will drop by a corresponding factor of $G_i$.

We now place the transformed cubes by stacking them along the first
coordinate, sorted by increasing
$f^{(\boldsymbol{x}_i,\boldsymbol{\theta})}_{\theta_2}(x)/f^{(\boldsymbol{x}_i,\boldsymbol{\theta})}_{\theta_1}(x)$.

Notably, because the probability density $f_{\theta_1}$
in all transformed cubes is $G_i/G_i=1$, it is
possible to arrange all transformed cubes in this way so that, together, they
fill exactly the unit cube in $\mathbb{R}^N$.
Let the new estimation problems created in this way be
$(\boldsymbol{x}'_i,\boldsymbol{\theta})$,
let $t_i:\mathbb{R}^N\to\mathbb{R}^N$ be the transformation,
$t_i(\boldsymbol{x}_i)=\boldsymbol{x}'_i$, applied on
the observation space and let $t^1_i(x)$ be the first coordinate value of
$t_i(x)$.

By IRO, $L(f^{(\boldsymbol{x}'_i,\boldsymbol{\theta})}_{\theta_0},f^{(\boldsymbol{x}'_i,\boldsymbol{\theta})}_{\theta_1})=L(f^{(\boldsymbol{x}_i,\boldsymbol{\theta})}_{\theta_0},f^{(\boldsymbol{x}_i,\boldsymbol{\theta})}_{\theta_1})$,
which we know tends to $L(p,q)$.

Consider the probability density of each
$f^{(\boldsymbol{x}'_i,\boldsymbol{\theta})}_\theta$ over its support
$[0,1]^N$. This is a probability density that is uniform along all axes except
the first, but has some marginal, $s=s^i_\theta$, along the first axis.
We denote such a distribution by $D_N(s)$. Specifically, for $\theta=\theta_2$,
because of our choice of sorting order, we have
$s^i_{\theta_2}=c[f^{(\boldsymbol{x}_i,\boldsymbol{\theta})}_{\theta_2},f^{(\boldsymbol{x}_i,\boldsymbol{\theta})}_{\theta_1}]$,
so by Lemma~\ref{L:ccont}, this is known to $\mathcal{M}$-converge to
$c[f^{(\boldsymbol{x},\boldsymbol{\theta})}_{\theta_2},f^{(\boldsymbol{x},\boldsymbol{\theta})}_{\theta_1}]$.

If $M=1$, the above is enough to show that the $\mathcal{M}$-limit problem
of $(\boldsymbol{x}'_i,\boldsymbol{\theta})$ exists. If $M>1$, consider
the following.

Let $t:X\to [0,1]$ be the transformation mapping each $x\in X$ to
the supremum $\tilde{t}$ for which
$c[f^{(\boldsymbol{x},\boldsymbol{\theta})}_{\theta_2},f^{(\boldsymbol{x},\boldsymbol{\theta})}_{\theta_1}](\tilde{t})\le f^{(\boldsymbol{x},\boldsymbol{\theta})}_{\theta_2}(x)/f^{(\boldsymbol{x},\boldsymbol{\theta})}_{\theta_1}(x)$.
This will be satisfied with equality wherever
$c[f^{(\boldsymbol{x},\boldsymbol{\theta})}_{\theta_2},f^{(\boldsymbol{x},\boldsymbol{\theta})}_{\theta_1}]$
is continuous, which (because it is monotone) it is in all but a measure
zero of the $\tilde{t}$, and therefore of the $x$.

Thus, in all but a diminishing measure of $x$ we have that the value of
$f^{(\boldsymbol{x}_i,\boldsymbol{\theta})}_{\theta_2}(x)/f^{(\boldsymbol{x}_i,\boldsymbol{\theta})}_{\theta_1}(x)$ approaches $f^{(\boldsymbol{x},\boldsymbol{\theta})}_{\theta_2}(x)/f^{(\boldsymbol{x},\boldsymbol{\theta})}_{\theta_1}(x)$, which in turn equals the value $c[f^{(\boldsymbol{x},\boldsymbol{\theta})}_{\theta_2},f^{(\boldsymbol{x},\boldsymbol{\theta})}_{\theta_1}](t(x))$.
On the other hand, we have that
$s^i_{\theta_2}=c[f^{(\boldsymbol{x}_i,\boldsymbol{\theta})}_{\theta_2},f^{(\boldsymbol{x}_i,\boldsymbol{\theta})}_{\theta_1}]$ also $\mathcal{M}$-approaches
$c[f^{(\boldsymbol{x},\boldsymbol{\theta})}_{\theta_2},f^{(\boldsymbol{x},\boldsymbol{\theta})}_{\theta_1}]$
by Lemma~\ref{L:ccont}, and satisfies
\[
f^{(\boldsymbol{x}_i,\boldsymbol{\theta})}_{\theta_2}(x)/f^{(\boldsymbol{x}_i,\boldsymbol{\theta})}_{\theta_1}(x)=s^i_{\theta_2}(t^1_i(x)).
\]
Together, this indicates $\left(c[f^{(\boldsymbol{x},\boldsymbol{\theta})}_{\theta_2},f^{(\boldsymbol{x},\boldsymbol{\theta})}_{\theta_1}](t^1_i(x))\right) \xrightarrow{\mathcal{M}} c[f^{(\boldsymbol{x},\boldsymbol{\theta})}_{\theta_2},f^{(\boldsymbol{x},\boldsymbol{\theta})}_{\theta_1}](t(x))$.

Because
$c[f^{(\boldsymbol{x},\boldsymbol{\theta})}_{\theta_2},f^{(\boldsymbol{x},\boldsymbol{\theta})}_{\theta_1}]$
is monotone strictly increasing,
it follows that $t^1_i(x)$ $\mathcal{M}$-converges to $t(x)$.
For all other $\theta\in [0,1]^M$ this then implies that $s^i_\theta$
$\mathcal{M}$-converges to $c[f^{(\boldsymbol{x},\boldsymbol{\theta})}_{\theta},f^{(\boldsymbol{x},\boldsymbol{\theta})}_{\theta_1}]$,
because by construction all the problem's $r_{f_{\theta},q}$ are monotone
increasing with each other.

All $f^{(\boldsymbol{x}'_i,\boldsymbol{\theta})}_{\theta}$ therefore have a
limit, that limit being 
$D_N(c[f^{(\boldsymbol{x},\boldsymbol{\theta})}_{\theta},f^{(\boldsymbol{x},\boldsymbol{\theta})}_{\theta_1}])$.

In particular, the limit at $\theta=\theta_0$ is $D_N(c[p,q])$ and the limit
at $\theta=\theta_1$ is $U([0,1]^N)$, the uniform distribution over the
unit cube.

By problem-continuity of $L$,
$L(D_N(c[p,q]),U([0,1]^N))=L(p,q)$.
Hence, $L(p,q)$ is a function of only $c[p,q]$.
\end{proof}

For a well-behaved continuous estimation problem
$(\boldsymbol{x},\boldsymbol{\theta})$,
if for every $\theta\in\Theta$ the function
$L_{\theta}(\theta')\defeq L_{(\boldsymbol{x},\boldsymbol{\theta})}(\theta',\theta)$ has all its second derivatives
at $\theta'=\theta$ (a condition that is true for every well-behaved $L$),
denote its Hessian matrix at the $\theta'=\theta$ by
$H_L^{\theta}=H_L^{\theta}[(\boldsymbol{x},\boldsymbol{\theta})]$.

\begin{lemma}\label{L:Fisher}
Let $(\boldsymbol{x},\boldsymbol{\theta})$ be a well-behaved
$\boldsymbol{\theta}$-continuous estimation problem,
and let $L$ be a well-behaved likelihood-based loss function
satisfying that $L(P,Q)$ is a function of only $c[P,Q]$.

There exists a nonzero constant $\gamma$, dependent only on
the choice of $L$, such that for every $\theta\in\Theta$
the Hessian matrix $H_L^{\theta}$ equals
$\gamma$ times the Fisher information matrix $I_{\theta}$.
\end{lemma}

\begin{proof}
We wish to calculate the derivatives of $L(P_{\theta_1},P_{\theta_2})$
according to $\theta_1$. For convenience, let us define a new function,
$L_Q$, which describes $L$ in a one-parameter form, by
\[
L_Q(r_{P,Q})=L(P,Q).
\]
This is possible by our assumption that $L(P,Q)$ is only a function of $c[P,Q]$.

We differentiate $L_Q(r_{P,Q})$ as we would any composition of functions.
The derivatives of $r_{P_{\theta_1},P_{\theta_2}}$ in $\theta_1$ are
straightforward to compute, so we
concentrate on the question of how minute perturbations of $r$ affect
$L_Q(r)$.

For this, we first extend the domain of $L_Q(r)$. Natively, $L_Q(r)$ is
only defined when $\mathbf{E}_{\boldsymbol{x}\sim Q}(r(\boldsymbol{x}))=1$.
However, to be able to perturb $r$ more freely, we define, for finite
expectation $r$,
$L_Q(r)=L_Q(r/\mathbf{E}_{\boldsymbol{x}\sim Q}(r(\boldsymbol{x})))$.

Let $Y\subseteq X$ be a set with
$\mathbf{P}_{\boldsymbol{x}\sim Q}(\boldsymbol{x}\in Y)=\epsilon>0$ such that
for all $x\in Y$, $r(x)$ is a constant, $r_0$. The derivative of $L_Q(r)$ in
$Y$ is defined as
\[
\left[\nabla_Y(L_Q)\right](r)=\lim_{\Delta\to 0} \frac{L_Q(r+\Delta \cdot\chi_Y)-L_Q(r)}{\Delta},
\]
where for any $S\subseteq\mathbb{R}^s$, we denote by $\chi_S$ the function
over $\mathbb{R}^s$ that yields $1$ when the input is in $S$ and $0$ otherwise.

By our smoothness assumption on well-behaved $L$, this derivative is known to
exist,
because it is straightforward to construct a well-behaved continuous
estimation problem for which these would be (up to normalisation) the
$r_{P_{\theta_0},P_{\theta_0+\Delta e_1}}$ values for
some $\theta_0$ and basis vector $e_1$.

Consider, now, what this derivative's value can be. By assumption that $L(P,Q)$
only depends on $c[P,Q]$,
we know that the derivative's value can depend on $\epsilon$, $r_0$ and
the function $r$, but not on any other properties of $Y$, because any such
$Y$ will yield the same $c[P,Q]$ values throughout the calculation of
$\left[\nabla_Y(L_Q)\right](r)$. Hence, we can describe it as
$\Delta_{\text{x}}(\epsilon,r_0|r)$.

Consider, now, partitioning $Y$ into $2$ sets, each
of measure $\epsilon/2$ in $Q$.\footnote{This is straightforward to do in the
continuous case. If $X$ is discrete, one can do this by utilising ISI and
first translating $\boldsymbol{x}$ to $(\boldsymbol{x},\boldsymbol{y})$, where
$\boldsymbol{y}\sim \text{Bin}[n=1,p=1/2]$, where $\text{Bin}$ is the
binomial distribution.} The marginal impact of each set
on the value of $L$ is $\Delta_{\text{x}}(\epsilon/2,r_0|r)$, but their total
impact is $\Delta_{\text{x}}(\epsilon,r_0|r)$.
More generally, we can describe $\Delta_{\text{x}}(\epsilon,r_0|r)$ as
$\epsilon \cdot \Delta_{\text{x}}(r_0|r)$.

Utilising $L(\theta_1,\theta_2)$'s representation as a composition of
functions $L_Q(r)$, where $Q=P_{\theta_2}$ and
$r=r_{P_{\theta_1},P_{\theta_2}}$, and
noting that $\epsilon$ was, in the calculation above, the measure of $Y$ in
$Q=P_{\theta_2}$,
we can now write the first derivative of $L$ in some direction $i$ of
$\theta_1$ explicitly as an integral in this measure, i.e.\ in
``$\text{d}P_{\theta_2}$''.

For clarity of presentation, we will write this as an integral in
``$f_{\theta_2}(x)\text{d}x$'', using here and throughout the remainder of
the proof pdf notation, as would be appropriate when $\boldsymbol{x}$ is
known to be continuous. This change is meant merely to simplify the notation,
and in no way restricts the proof. Readers are welcome to verify that all
steps are equally valid for any $P_{\theta_1}$ and $P_{\theta_2}$
distributions.

If $c[f_{\theta_1},f_{\theta_2}]$ is a piecewise-constant
function, the derivative of $L$ can be written as follows.
\[
\frac{\partial L(\theta_1,\theta_2)}{\partial \theta_1(i)}=
\int_{X} \Delta_{\text{x}}(r_{f_{\theta_1},f_{\theta_2}}(x)|r_{f_{\theta_1},f_{\theta_2}}) \frac{\partial r_{f_{\theta_1},f_{\theta_2}}(x)}{\partial \theta_1(i)}f_{\theta_2}(x)\text{d}x.
\]

The same reasoning can be used to describe the second derivative of $L$
(this time in the directions $i$ and $j$ of $\theta_1$). The second derivative
of $L_q(r)$ when perturbing $r$ relative to two subsets $Y_1$ and $Y_2$ is
defined as
\[
\lim_{\Delta\to 0} \frac{\left[\nabla_{Y_1}(L_q)\right](r+\Delta\cdot\chi_{Y_2})-\left[\nabla_{Y_1}(L_q)\right](r)}{\Delta},
\]
and once again using Lemma~\ref{L:r} and the same symmetry, we can see that if
$Y_1$ and $Y_2$
are disjoint, if the measures of $Y_1$ and $Y_2$ in $q$ are, respectively,
$\epsilon_1$ and $\epsilon_2$, both positive,
and if the value of $r(x)$ for $x$ values in each subset
is a constant, respectively $r_1$ and $r_2$, then any such $Y_1$ and $Y_2$
will perturb $L_q(r)$ in exactly the same amount.
We name the second derivative coefficient
in this case $\Delta_{\text{xy}}(r_1,r_2|r)$.

The caveat that $Y_1$ and $Y_2$ must be disjoint is important, because if
$Y=Y_1=Y_2$ the symmetry no longer holds. This is a second case, and for it
we must define a different coefficient $\Delta_{\text{xx}}(r_0|r)$, where
$r_0=r_1=r_2$.

In the case where $c[f_{\theta_1},f_{\theta_2}]$ is
a piecewise-constant function, the second derivative,
$\frac{\partial^2  L(\theta_1,\theta_2)}{\partial \theta_1(i)\partial \theta_1(j)}$,
can therefore be written as
\begin{equation}\label{Eq:second_deriv}
\begin{aligned}
&\int_{X} \Delta_{\text{x}}(r_{f_{\theta_1},f_{\theta_2}}(x)|r_{f_{\theta_1},f_{\theta_2}}) \frac{\partial^2 r_{f_{\theta_1},f_{\theta_2}}(x)}{\partial \theta_1(i)\partial \theta_1(j)}f_{\theta_2}(x)\text{d}x \\
&\quad+\int_{X}\int_{X} \Delta_{\text{xy}}(r_{f_{\theta_1},f_{\theta_2}}(x_1),r_{f_{\theta_1},f_{\theta_2}}(x_2)|r_{f_{\theta_1},f_{\theta_2}}) \\
&\quad\quad\quad\quad\quad\quad\quad\quad\frac{\partial r_{f_{\theta_1},f_{\theta_2}}(x_1)}{\partial \theta_1(i)}\frac{\partial r_{f_{\theta_1},f_{\theta_2}}(x_2)}{\partial \theta_1(j)}f_{\theta_2}(x_2)\text{d}x_2 f_{\theta_2}(x_1)\text{d}x_1 \\
&\quad+\int_{X} \Delta_{\text{xx}}(r_{f_{\theta_1},f_{\theta_2}}(x)|r_{f_{\theta_1},f_{\theta_2}})\frac{\partial r_{f_{\theta_1},f_{\theta_2}}(x)}{\partial \theta_1(i)}\frac{\partial r_{f_{\theta_1},f_{\theta_2}}(x)}{\partial \theta_1(j)}f_{\theta_2}(x)\text{d}x.
\end{aligned}
\end{equation}

In order to calculate $H_L^{\theta}$, consider $H_L^{\theta}(i,j)$. This
equals
$\frac{\partial^2  L(\theta_1,\theta_2)}{\partial \theta_1(i)\partial \theta_1(j)}$
where $\theta_1=\theta_2=\theta$. In particular, $c[f_{\theta_1},f_{\theta_2}]$
is $\chi_{(0,1]}$ and $r_{f_{\theta_1},f_{\theta_2}}$ is $\chi_{X}$.

The value of \eqref{Eq:second_deriv} in this case becomes
\begin{equation}\label{Eq:second_deriv_at_1}
\begin{aligned}
&\Delta_{\text{x}}\left(1\middle|\chi_{X}\right)\int_{X} \left.\frac{\partial^2 r_{f_{\theta_1},f_{\theta}}(x)}{\partial \theta_1(i)\partial \theta_1(j)}\right\rvert_{\theta_1=\theta}f_{\theta}(x)\text{d}x \\
&\quad +\Delta_{\text{xy}}\left(1,1\middle|\chi_{X}\right)\left(\int_{X} \left.\frac{\partial r_{f_{\theta_1},f_{\theta}}(x)}{\partial \theta_1(i)}\right\rvert_{\theta_1=\theta}f_{\theta}(x)\text{d} x\right) \\
&\quad\quad\quad\quad\quad\quad\quad\quad\left(\int_{X} \left.\frac{\partial r_{f_{\theta_1},f_{\theta}}(x)}{\partial \theta_1(j)}\right\rvert_{\theta_1=\theta}f_{\theta}(x)\text{d} x\right)\\
&\quad +\Delta_{\text{xx}}\left(1\middle|\chi_{X}\right)\int_{X} \left(\left.\frac{\partial r_{f_{\theta_1},f_{\theta}}(x)}{\partial \theta_1(i)}\right\rvert_{\theta_1=\theta}\right)\left(\left.\frac{\partial r_{f_{\theta_1},f_{\theta}}(x)}{\partial \theta_1(j)}\right\rvert_{\theta_1=\theta}\right)f_{\theta}(x)\text{d}x.
\end{aligned}
\end{equation}

Note, however, that because $(\boldsymbol{x},\boldsymbol{\theta})$ is an
estimation problem, i.e.\ all its likelihoods are probability measures, not
general measures, it is the case that
\[
\int_{X} r_{f_{\theta_1},f_{\theta}}(x)f_{\theta}(x)\text{d}x
=\int_{X} f_{\theta_1}(x)\text{d}x=1,
\]
and is therefore a constant
independent of either $\theta_1$ or $\theta$. Its various derivatives in
$\theta_1$ are accordingly all zero. This makes the first two summands in
\eqref{Eq:second_deriv_at_1} zero. What is left, when setting
$\gamma=\Delta_{\text{xx}}(1|\chi_{X})$, is
\begin{align*}
H_L^{\theta}(i,j)
&=\gamma \int_{X} \left(\left.\frac{\partial r_{f_{\theta_1},f_{\theta}}(x)}{\partial \theta_1(i)}\right\rvert_{\theta_1=\theta}\right)\left(\left.\frac{\partial r_{f_{\theta_1},f_{\theta}}(x)}{\partial \theta_1(j)}\right\rvert_{\theta_1=\theta}\right)f_{\theta}(x)\text{d}x\\
&=\gamma \int_X \left(\left.\frac{\partial f_{\theta_1}(x)/f_{\theta}(x)}{\partial \theta_1(i)}\right\rvert_{\theta_1=\theta}\right)\left(\left.\frac{\partial f_{\theta_1}(x)/f_{\theta}(x)}{\partial \theta_1(j)}\right\rvert_{\theta_1=\theta}\right) f_{\theta}(x)\text{d}x\\
&=\gamma \int_X \left(\frac{\partial \log f_{\theta}(x)}{\partial \theta(i)}\right)\left(\frac{\partial \log f_{\theta}(x)}{\partial \theta(j)}\right) f_{\theta}(x)\text{d}x\\
&=\gamma \mathbf{E}_{\boldsymbol{x}\sim f_{\theta}}\left(\left(\frac{\partial \log f_{\theta}(\boldsymbol{x})}{\partial \theta(i)}\right)\left(\frac{\partial \log f_{\theta}(\boldsymbol{x})}{\partial \theta(j)}\right)\right)\\
&=\gamma I_{\theta}(i,j).
\end{align*}

Hence, $H_L^{\theta}=\gamma I_{\theta}$.\footnote{If $\boldsymbol{x}$ is
discrete, the final result would have used probabilities rather than
probability densities. This is consistent, however, with the way Fisher
information is defined in this more general case. In fact, some sources
(e.g., \citep{bobkov2014fisher}) define the Fisher information directly from
Radon-Nikodym derivatives.}

As a final point in the proof, we remark that $\gamma$
must be nonzero, because if it had been zero, $H_L^{\theta}$
would have been zero for every $\theta$ in every
well-behaved continuous/$\boldsymbol{\theta}$-continuous estimation problem,
contrary to our sensitivity assumption on well-behaved loss functions.
\end{proof}

\begin{lemma}\label{L:Hessian}
If $(\boldsymbol{x},\boldsymbol{\theta})$ is a well-behaved
$\boldsymbol{\theta}$-continuous estimation
problem, $L$ is a well-behaved loss function that is discriminative for it,
and $A$ is an attenuation function, and
if, further, for every $\theta\in\Theta$, $H_L^{\theta}$ is a positive
definite matrix, define
\[
\hat{\theta}(x)\defeq\argmax_{\theta} \frac{f(\theta|x)}{\sqrt{|H_L^{\theta}|}}.
\]

For every $x$, $\hat{\theta}_{L,A}(x)$ is a non-empty subset of
$\hat{\theta}(x)$, where $\hat{\theta}_{L,A}$ is the risk-averse estimator
defined over $L$ and $A$.
In particular, if $\hat{\theta}$ is a well-defined point
estimator for the problem, then $\hat{\theta}_{L,A}=\hat{\theta}$.
\end{lemma}

\begin{proof}
When calculating
\begin{equation}\label{Eq:AkL}
\int_\Theta f(\theta'|x) A(kL(\theta',\theta)) \text{d}\theta'
\end{equation}
for asymptotically large $k$ values
one only needs to consider the integral over $B(\theta,\epsilon)$, the ball of
radius $\epsilon$ around $\theta$, for any $\epsilon>0$, as for a sufficiently
large $k$, the rest of the integral values will be zero by the
discriminativity assumption.

By definition of $H_L^{\theta}$, the second order Taylor approximation for
$L(\theta',\theta)$ around $\theta$ is
\begin{equation}\label{Eq:Taylor}
L(\theta',\theta)=\frac{1}{2}(\theta'-\theta)^T H_L^{\theta} (\theta'-\theta)\pm \frac{m}{6} |\theta'-\theta|^3,
\end{equation}
where $m$ is the maximum absolute third derivative of $L$ over all
$\Theta\times\Theta$ and all possible differentiation directions.
This maximum exists because the third derivative is continuous and
$\Theta\times\Theta$ is compact.

As $k$ grows to infinity, the value of \eqref{Eq:AkL} therefore tends to
\[
f(\theta|x)\int_{B(\theta,\epsilon)} A\left(\frac{k}{2}(\theta'-\theta)^T H^\theta_L (\theta'-\theta)\right)\text{d}\theta',
\]
which can be computed via a Jacobian transformation as
\begin{equation}\label{Eq:prop}
\frac{f(\theta|x)}{\sqrt{|H^\theta_L|}}\left(\frac{2^{M/2}\int_{\mathbb{R}^M} A(|\omega|^2)\text{d}\omega}{\sqrt{k^M}}\right),
\end{equation}
where the parenthesised expression on the right is a multiplicative factor
independent of $\theta$.

We have therefore shown for $F_x(\theta)\defeq f(\theta|x)/\sqrt{|H^\theta_L|}$
that
\begin{equation}\label{Eq:maxlim}
\argmax_{\theta\in\Theta} \lim_{k\to\infty} \sqrt{k^M} \int_\Theta f(\theta'|x) A(kL(\theta',\theta)) \text{d}\theta' = \argmax_{\theta\in\Theta} F_x(\theta)=\hat{\theta}(x).
\end{equation}

What we are trying to compute, however, is
\begin{equation}\label{Eq:LA}
\hat{\theta}_{L,A}(x)=\setlim_{k\to\infty}\argmax_{\theta\in\Theta} \int_\Theta f(\theta'|x) A(kL(\theta',\theta)) \text{d}\theta'.
\end{equation}
The multiplicative difference of $\sqrt{k^M}$ between \eqref{Eq:LA} and
\eqref{Eq:maxlim} is immaterial, as its addition into \eqref{Eq:LA} would not
have changed the argmax value, but the reversal in the order of the
quantifiers can, at least potentially, change the result.

In order to show that no value other than a maximiser of $F_x$ can be part of
$\hat{\theta}_{L,A}(x)$,
we need to prove that outside of any neighbourhood of a maximiser of $F_x$,
$F_x$ is bounded from above away from its maximum, and that the rates of
convergence in $k$ over all $\theta$ are uniformly bounded.

The fact that $\theta$ values outside of a neighbourhood of the maximisers
of $F_x$ are bounded away from the maximum is due to $F_x$ being continuous
over a compact space.  Let us therefore bound the convergence rates of
\eqref{Eq:AkL} as $k$ goes to infinity.

The value of \eqref{Eq:AkL} is equal to that of \eqref{Eq:prop} up to a
multiplicative factor to do with the changes in $f(\theta'|x)$ and the changes
in
\begin{equation}\label{Eq:Lratio}
\frac{A(kL(\theta',\theta))}{A\left(\frac{k}{2}(\theta'-\theta)^T H_L^\theta (\theta'-\theta)\right)}
\end{equation}
over the volume of the integral, with the effective
volume of the integral being determined by
$\{\theta' : L(\theta',\theta)\le a_0/k\}$, where $a_0$ is the threshold
value of $A$. This, in turn, bounds
$|\theta'-\theta|$ within the effective volume of the integral as a function
of $H_{\text{min}}$, the minimum eigenvalue
of $H_L^{\theta'}$ for any $\theta'\in\Theta$
(which is known to be positive, because $H_L^{\theta'}$ was assumed to be a
positive definite matrix for all $\theta'$ and its minimum eigenvalue is a
continuous function of $\theta'$ over the compact space $\Theta$).

Equation \eqref{Eq:Taylor} allows us to bound \eqref{Eq:Lratio},
by means of $m$ and of $A'_{\text{max}}$, the maximum absolute derivative of
$A$, while the
bound on $|\theta'-\theta|$ allows us to bound $f(\theta'|x)$ similarly, as
$f(\theta|x)\pm f'_{\text{max}} |\theta'-\theta|$, where $f'_{\text{max}}$
is the global maximum absolute derivative of $f(\theta'|x)$ over $\theta'$ at
$x$, for any $\theta'\in\Theta$ and in any direction.

The four elements that globally bound the speed of convergence are therefore
$m$, $H_{\text{min}}$, $A'_{\text{max}}$ and $f'_{\text{max}}$, all of which
are finite, positive numbers, because they are extreme values of continuous,
positive functions over a compact space, so the
convergence rates are all uniformly bounded, as required.

As a last point, we remark that the setlim of \eqref{Eq:LA} is calculated
entirely on subsets of the compact set $\Theta$, so by the
Bolzano-Weierstrass theorem \citep{bartle2011introduction}
$\hat{\theta}_{L,A}$ cannot be the empty set.
\end{proof}

We now turn to prove our main claim.

\begin{proof}[Proof of Theorem~\ref{T:main}]
By Lemma~\ref{L:likelihood} we know $L$ to be a likelihood-based loss
function, and by Lemma~\ref{L:r} we know its value at $L(P,Q)$ depends only
on the value of $c[P,Q]$. With these prerequisites, we can use
Lemma~\ref{L:Fisher} to conclude that up to a nonzero constant
multiple $\gamma$
its $H_L^{\theta}$ is the Fisher information matrix $I_{\theta}$. This,
in turn, we've assumed to be positive definite by requiring
$\hat{\theta}_\text{WF}$ to be well-defined.

Furthermore, $\gamma$ cannot be negative, as in combination with a positive
definite Fisher information matrix this would indicate that $H_L^{\theta}$
is not positive semidefinite, causing $L$ to attain negative values in the
neighbourhood of $\theta$.

It is therefore the case that $H_L^{\theta}$ must be positive definite, and
Lemma~\ref{L:Hessian} can be used to conclude the correctness
of the theorem.
\end{proof}

\subsection{Feasibility and necessity}\label{S:contfnn}

By convention, when using the ax\-iomat\-ic approach one also shows that
the assumptions taken are all feasible, and that all axioms are necessary.
We do so, in the context of continuous problems, in this section.

\subsubsection{Feasibility}\label{S:feasibility}

An $f$-divergence \citep{ali1966general} is a loss function $L$ that can be
computed as
\[
L(p,q)=\int_X F(p(x)/q(x)) q(x) \text{d}x.
\]
We call $F$ the $F$-function of the $f$-divergence (refraining from using
the more common term ``$f$-function'' so as to avoid unnecessary
confusion with our probability density function).
It should be convex and satisfy $F(1)=0$.

Any $f$-divergence whose $F$-function has 3 continuous derivatives and
satisfies $F''(1)>0$ meets most of our requirements
regarding a well-behaved $L$ function. The one outstanding requirement is that
of $\mathcal{M}$-continuity. In the very general case discussed here, where,
for example, $f(x,\theta)$
can diverge to infinity over $x$, there is a risk that for some estimation
problems under some $L$ functions a subset of $X$ of diminishing measure can
have a non-negligible impact on the value of $L(p,q)$, for some $p$ and $q$.
In order to guarantee $\mathcal{M}$-continuity for all estimation problem
sequences, we further require our chosen $f$-divergence's $F$-function to not be
in absolute value super-linear in its input, i.e.\ for its absolute value,
$|F(r)|$, to be upper-bounded by some linear function $Ar+B$. When this is the
case,
\begin{align*}
|L(p,q)|&=\left|\int_X F(p(x)/q(x)) q(x) \text{d}x\right|\le\int_X |F(p(x)/q(x))| q(x) \text{d}x \\
&\le\int_X\left(A\frac{p(x)}{q(x)}+B\right) q(x) \text{d}x= A+B<\infty,
\end{align*}
so $\mathcal{M}$-continuity holds.

An example of a commonly-used $L$ function satisfying all criteria for
well-behavedness is squared Hellinger distance \citep{pollard2002user},
\[
H^2(p,q)=\frac{1}{2}\int_X \left(\sqrt{p(x)}-\sqrt{q(x)}\right)^2\text{d} x,
\]
which is the $f$-divergence whose $F$-function is $F(r)=1-\sqrt{r}$.

\subsubsection{Necessity of IRP}

A well-known loss function that satisfies all axioms except IRP is
quadratic loss,
$L(\theta_1,\theta_2)=|\theta_1-\theta_2|^2$.
A risk-averse estimator with this loss function yields the
$f$-MAP estimate.

A more involved example for this is the loss function
$L(\theta_1,\theta_2)=\sqrt[M]{f(\theta_2)^2} |\theta_1-\theta_2|^2$.
From Lemma~\ref{L:Hessian}, we can deduce that this loss function yields the
Maximum Likelihood (ML) estimate. This is telling, because while the loss
function satisfies all axioms except IRP, the ML estimate itself is invariant
to the representation of the problem's parameter space. In this sense, ML
(and, as a corollary, certain penalised ML estimators) can be said to satisfy
all our axioms on the level of the overall behaviour of the estimate, while WF
remains unique in the fact that it is the only risk-averse estimate
\emph{whose loss function} satisfies all axioms. It is therefore important
to recognise the loss function as the antecedent, and
the behaviour of the estimate as a whole as its consequence.

\subsubsection{Necessity of IRO}

For a loss function satisfying all axioms except IRO,
let $L(p,q)=\int_X q(x) (p(x)-q(x))^2 \text{d}x$, which is the expected
square difference between the probability densities at $\boldsymbol{x}\sim q$.
The risk-averse estimator over $L$ is described by Lemma~\ref{L:Hessian}.

Calculating $H_L^\theta$ we get
\[
H_L^\theta(i,j)=\mathbf{E}_{\boldsymbol{x}\sim f_\theta}\left(2\left(\frac{\partial f_\theta(\boldsymbol{x})}{\partial \theta(i)}\right)\left(\frac{\partial f_\theta(\boldsymbol{x})}{\partial \theta(j)}\right)\right),
\]
which is different to the Fisher information matrix, and defines an
estimator that is not WF.

\subsubsection{Necessity of IIA}

We construct a loss function $L$ that satisfies all axioms except IIA as
follows.
Let $L_1$ and $L_2$ be two well-behaved loss functions satisfying all axioms
(such as, for example, two well-behaved $f$-divergences)
and let $t$ be a threshold value.

Consider the function
\[
P(\theta)=\mathbf{P}(L_1(\boldsymbol{\theta},\theta)\le t).
\]

By construction, this function is independent of representation.

Define $L(\theta_1,\theta_2)=P(\theta_2)^2 L_2(\theta_1,\theta_2)$.

By Lemma~\ref{L:Hessian}, the resulting risk-averse estimator will equal
\[
\argmax_{\theta\in\Theta} \frac{f(\theta|x)}{\sqrt{|H_L^\theta|}}=\argmax_{\theta\in\Theta} \frac{f(\theta|x)}{P(\theta)^M\sqrt{|H_{L_2}^\theta|}},
\]
because $P(\theta_2)^2$ is independent of $\theta_1$ and therefore acts as a
constant multiplier in the calculation of the Hessian.

This new estimator is different to the Wallace-Freeman estimator in the fact
that it adds a weighing factor $P(\theta)^M$. Any other weighing factor can
similarly be added, given that it is a function of $\theta$ that is
independent of representation and satisfies the well-behavedness criteria of
a loss function.

\section{Parameter estimation based on finite information}

The final case remaining is that of estimation problems with a continuous
$\boldsymbol{\theta}$ but a discrete $\boldsymbol{x}$. We will refer to such
problems as \emph{semi-continuous}. Semi-continuous problems are both a common
case and an important one. They are common because we often wish to estimate
continuous parameters based on finite information. They are important because
this case is central in MML theory \citep{Wallace2005}, which is the only
context in which the Wallace-Freeman estimate has previously been in use.

The axioms IRP, IRO and IIA, which sufficed for continuous problems, are not
sufficient to fully characterise a solution in the semi-continuous case.
To demonstrate this, consider the classic example of estimating the
$p$ parameter in a binomial distribution, for a known $n$.
For simplicity of presentation, instead of observing $\text{Bin}(n,p)$,
we will assume our observations are $n$ Bernoulli trials.
The observation space is therefore $X=\{0,1\}^n$.

In addition, we modify the classic problem slightly by limiting our parameter
space to only $\Theta=[\epsilon,1/2]$ for some small $\epsilon>0$,
rather than its full possible range.
We assume that the prior for $p$ is uniform within $\Theta$.

We will show regarding estimators of the form
\begin{equation}\label{Eq:genloss}
\hat{\theta}(x)=\argmax_{\theta\in\Theta} f(\theta|x)/F(\theta),
\end{equation}
where $F$ can be any twice continuously differentiable function into
$\mathbb{R}^{+}$, that they can be described as risk-averse estimators over
loss functions satisfying all of IRP, IRO and IIA.

To construct a loss function that will satisfy \eqref{Eq:genloss}, let
$G(\theta)=\int_{\epsilon}^{\theta} F(\theta') \text{d}\theta'$ and for any
$p\in \Theta$ consider the likelihood $P_p$, i.e.\ the distribution of
$\boldsymbol{x}$ given $p$. This equals
\[
P_p(x)= \prod_{i=1}^{n} p^{x(i)} (1-p)^{1-x(i)},
\]
where $x(i)$ is the result of the $i$'th Bernoulli trial.

Consider, now, that for all $p$,
$\prod_{x'\in X} P_p(x') = \left(p(1-p)\right)^{n 2^{n-1}}$.
This allows us to retrieve the original value of $p$ from the distribution as
\[
p=\tilde{p}(P_p)\defeq\frac{1}{2}-\sqrt{\frac{1}{4}-\left(\prod_{x'\in X} P_p(x')\right)^{\frac{1}{n 2^{n-1}}}}.
\]

We can now use as a likelihood-based loss function
\[
L(P_p,P_q)=\frac{1}{2}(G(\tilde{p}(P_p))-G(\tilde{p}(P_q)))^2.
\]
This loss function leads to \eqref{Eq:genloss} because $H_L^\theta$ equals
$F(\theta)^2$.

The key to resolving the semi-continuous case is Axiom ISI. This axiom
states, for example, that when estimating the binomial parameter, it should
not matter whether we observe $n$ Bernoulli trials or merely their sum.
The estimate should only depend on the sufficient statistic.

Our theorem for the semi-continuous case is as follows.

\begin{thm}\label{T:semicontinuous}
If $(\boldsymbol{x},\boldsymbol{\theta})$ is a well-behaved semi-continuous
estimation problem for which $\hat{\theta}_{\text{WF}}$ is a well-defined
set estimator, and if $L$ is a well-behaved loss function, discriminative
for $(\boldsymbol{x},\boldsymbol{\theta})$,
that satisfies all of IIA, IRP, IRO and ISI,
then any risk-averse estimator
$\hat{\theta}_{L,A}$ over $L$, regardless of its attenuation function $A$,
is a well-defined set estimator, and for every $x$,
\[
\hat{\theta}_{L,A}(x)\subseteq \hat{\theta}_{\text{WF}}(x).
\]

In particular, if $\hat{\theta}_{\text{WF}}$ is a well-defined point estimator,
then so is $\hat{\theta}_{L,A}$, and
\[
\hat{\theta}_{L,A}=\hat{\theta}_{\text{WF}}.
\]
\end{thm}

\begin{proof}
The proof is essentially identical to that of Theorem~\ref{T:main}. The only
change is that we can no longer apply Lemma~\ref{L:r}. However, we claim that
$L(P,Q)$ only depends on $c[P,Q]$ despite this, for which reason we can still
apply Lemma~\ref{L:Fisher} and Lemma~\ref{L:Hessian} as before to complete
the proof.

To show this, consider any two pairs $(P_1,Q_1)$, $(P_2,Q_2)$ such that
$c[P_1,Q_1]=c[P_2,Q_2]$. We claim that $L(P_1,Q_1)=L(P_2,Q_2)$, and prove this
as follows.

We begin by using ISI to map every $\boldsymbol{x}$ value, $x$, to
$(x,r_{P,Q}(x))$. Next, we use IRO to map $(x,r_{P,Q}(x))$ to
$(r_{P,Q}(x),x)$.

By the definition of $r_{P,Q}$, $x$ is independent of the choice of $P$ versus
$Q$, given that $r_{P,Q}$ is known.
For our purposes, however, we need to construct an entire semi-continuous
estimation problem for which this independence is satisfied for every choice
$\theta$ of $\boldsymbol{\theta}$, and not just for $P_{\theta}\in\{P,Q\}$.

Such a problem is not difficult to construct.
For example, if $M$, the dimension of the
parameter space, is $1$, one can define a problem where the likelihood of
$\boldsymbol{x}$ is $\gamma P + (1-\gamma) Q$, and one is required to estimate
the value of $\gamma\in[0,1]$. A similar construction for a general $M$ was
shown in the proof of Lemma~\ref{L:r}.

In this new estimation problem, one can use ISI to determine that the value
of $L(P,Q)$ will not change if the observations were to be mapped from
$(r_{P,Q}(x),x)$ to $r_{P,Q}(x)$.

The value of $L(P,Q)$ can therefore only depend on the total probability
assigned by $Q$ to $\boldsymbol{x}$ values, $x$, of each particular
$r_{P,Q}(x)$ value. (It can also depend on the value thus assigned by $P$,
but this can be calculated as the value assigned by $Q$ multiplied by
$r_{P,Q}(x)$, so is equivalent information.)

This is precisely the information conveyed by $c[P,Q]$.
\end{proof}

\subsection{Feasibility and necessity}

Most of what is needed in order to prove feasibility and necessity for the
semi-continuous case is identical to what was already done for the
continuous case in Section~\ref{S:contfnn}, with probability mass functions
replacing probability density functions and sums replacing integrals.
We will not repeat this here, where the original proofs hold
\textit{mutatis mutandis}.

In terms of feasibility, squared Hellinger distance (in its
discrete-probability formulation) remains a proof that our axioms are
feasible and that the WF estimate is a legitimate solution.

Our constructions for proving that IRP and IIA are both necessary hold, too,
and the necessity of ISI has been demonstrated with the example of
binomial estimation.

To demonstrate that IRO is necessary, let $x(1:k)$ be the value of $x$'s
first $k$ dimensions let $x(k)$ be the value of its
$k$'th dimension alone, and for a distribution $P$, let $P^y_k$ be the
distribution of $x(k)$ given that $x(1:k-1)=y$. Consider, now, a function
$L(P,Q)$ calculated in the following way, over another loss function, $L_1$.
\[
L(P,Q)=\sum_{k=1}^{N} \mathbf{E}_{\boldsymbol{x}\sim Q}\left[L_1(P^{\boldsymbol{x}(1:k-1)}_k,Q^{\boldsymbol{x}(1:k-1)}_k)\right],
\]
where $L_1$ is a loss function that only needs to be defined for estimation
problems where the dimensionality of the observation space is $1$.

By construction, this $L$ satisfies ISI for any $L_1$, because any dimension
that does not add information also does not add to the value of $L(P,Q)$.

The construction for an $L$ not satisfying IRO on continuous problems
now works as a construction for $L_1$ in the semi-continuous case, once
probability densities are changed to probability masses
and integrals to sums.

We have therefore shown that all axioms are necessary, that they are jointly
feasible, and that together they uniquely characterise WF in the
semi-continuous scenario.

\end{document}